\documentclass{article}
\PassOptionsToPackage{numbers}{natbib}
\usepackage{natbib}

\IfFileExists{GCSCS.sty}{
    \usepackage[preprint]{GCSCS}
    \usepackage[T1]{fontenc}
    \usepackage{iftex}
    \ifPDFTeX
      \usepackage[utf8]{inputenc}
    \fi
}{
    \usepackage[a4paper, top=2.5cm, bottom=2.5cm, left=2cm, right=2cm]{geometry}
    \usepackage{fontspec}
    \usepackage[english, bidi=basic, provide=*]{babel}
    \babelprovide[import, onchar=ids fonts]{english}
    \babelfont{rm}{Noto Serif}
}

\usepackage{xcolor}

\PassOptionsToPackage{hyphens}{url}
\usepackage{hyperref}
\usepackage{url}

\usepackage{booktabs}
\usepackage{amsmath}
\allowdisplaybreaks[1]
\usepackage{amsthm}
\usepackage{amssymb}
\newtheorem{theorem}{Theorem}
\newtheorem{definition}{Definition}
\newtheorem{remark}{Remark}

\usepackage{graphicx}
\usepackage{silence}
\usepackage[babel=false]{microtype}
\usepackage{subcaption}
\usepackage{float}
\usepackage{algorithm}
\usepackage{algorithmic}
\usepackage{wrapfig}
\newlength{\FigFailureWrapSep}
\usepackage{needspace}
\IfFileExists{placeins.sty}{\usepackage{placeins}}{\newcommand{\FloatBarrier}{\clearpage}}
\usepackage{marvosym}
\newcommand{\coremail}{\raisebox{0.05ex}{\textsuperscript{\footnotesize\Letter}}}

\AddToHook{begindocument/end}{\raggedbottom}

\title{Scaling Multi-Hop Training Data via\\[0.2em] Graph-Constrained Path Selection}

\author{%
  {\bfseries
  Pengyu Chen$^{1,2}$ \quad
  Yonggang Zhang$^{1,2}$ \quad
  Mingming Chen$^{1,2}$ \quad
  Jun Song$^{3}$
  } \\[0.4em]
  {\bfseries Wei Xue$^{1,2}$}\coremail \quad
  {\bfseries Yike Guo$^{1,2}$}\coremail \\[0.55em]
  $^{1}$The Hong Kong University of Science and Technology \\[0.2em]
  $^{2}$Hong Kong Generative AI Research and Development Center \\[0.2em]
  $^{3}$Hong Kong Baptist University
}

\begin{document}

\maketitle
\GCSCScorrespondence{%
  \textsuperscript{\footnotesize\Letter}\,Correspondence to: Wei Xue \textless weixue@ust.hk\textgreater, Yike Guo \textless yikeguo@ust.hk\textgreater.%
}

\begin{abstract}
Endowing large language models with compositional reasoning over specialized documents requires multi-hop training data at scale, where such data rarely exists outside of curated benchmarks built on structured sources. To construct it directly from plain, unannotated text, existing methods ask a single teacher model to jointly discover an evidence path through a document and verbalize it as a question-answer pair. However, these methods degrade sharply when documents are structured around repetitive templates and densely cross-referencing clauses, conditions that characterize most real-world specialized corpora. In this work, we decouple the two operations: reasoning paths are enumerated offline over a graph of contextual keyword centroids, and the teacher is invoked only to verbalize pre-validated paths. The graph enforces five geometric admissibility constraints, for which we provide Gram-matrix arguments establishing that local similarity bounds alone admit endpoint drift up to ${\sim}91^{\circ}$, and that an upper similarity bound is necessary to exit dense embedding cliques formed by boilerplate text. A matched-size ablation isolates the mechanism: at equal training scale, constrained and unconstrained chains yield indistinguishable downstream performance, and the gain at full scale comes from a 4.4$\times$ expansion of the usable corpus rather than from higher per-chain quality---reframing the role of graph constraints, in this setting, as raising teacher synthesizability rather than improving chain content. Fine-tuning Qwen3-32B on 80K examples constructed from the CUAD legal contract corpus improves closed-book Token F1 from 21.66\% to 38.58\%.
We have released our codes at \url{https://github.com/hkgai-official/GCSCS}.
\end{abstract}

\section{Introduction}
\label{sec:intro}
Large language models can recall isolated facts yet often fail to compose them into reliable multi-hop answers \cite{press2023measuring,ho2020constructing,geva2021did}. The gap is acute in legal contract analysis, where a useful answer may require linking a definition, an obligation, an exception, and a termination condition across different clauses \cite{hendrycks2021cuad}. Human-authored multi-hop chains exist for general-domain benchmarks \cite{yang2018hotpotqa,trivedi2022musique,talmor2018web}, but comparable domain-specific supervision is costly to produce \cite{hendrycks2021cuad}.

General synthetic instruction pipelines \cite{wang2023selfinstruct,xu2024wizardlm,ouyang2022training,guo2025synthetic,zhu2026CHIMERA,liu2026designer} require the teacher LLM to execute two tasks in a single prompt: discover connected evidence scattered across a document, and write the final question-answer pair. On straightforward domains this works reasonably well \cite{wang2023selfinstruct}. On legal contracts, however, the coupling proves brittle. Dense clause families, nested cross-references, and highly templated language create a large space of superficially similar but semantically distinct paths that must be navigated before generation even begins. Consequently, when we evaluate this joint operation on the Contract Understanding Atticus Dataset (CUAD) \cite{hendrycks2021cuad}, random multi-document chains pass a structured JSON quality gate only $22.0\%$ of the time, and the usable corpus saturates at approximately 18K examples regardless of generation budget \cite{hendrycks2021cuad}.

\Needspace{0.52\textheight}
\makeatletter
\@ifpackageloaded{lineno}{\nolinenumbers}{}
\makeatother
\begingroup
\setlength{\columnsep}{\FigFailureWrapSep}
\setlength{\belowcaptionskip}{-8pt}
\sloppy
\setlength{\emergencystretch}{3em}
\begin{wrapfigure}{r}{0.34\textwidth}
  \centering
  \includegraphics[width=\linewidth]{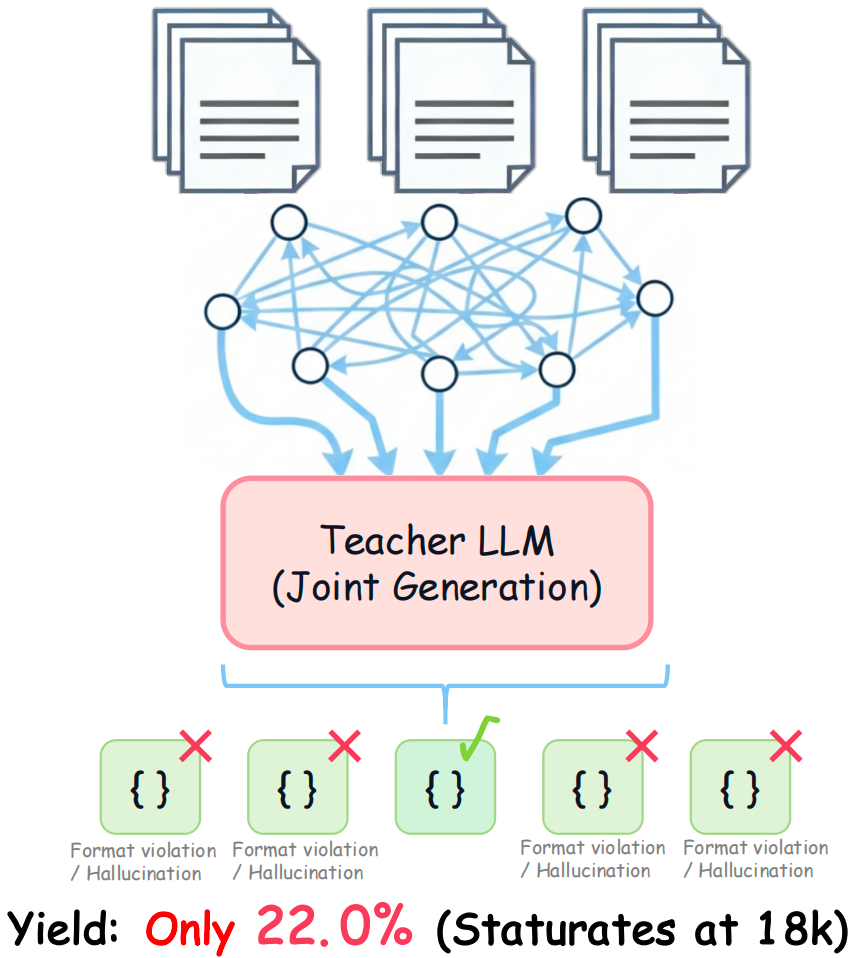}\par
  \captionsetup{font=small,labelfont={bf,small},skip=3pt,width=\linewidth,hypcap=false,margin=0pt}
  \caption{Illustration of the data yield bottleneck when coupling evidence discovery with language generation.}
  \label{fig:failure}
  \vspace{-10pt}
\end{wrapfigure}

We formalize this barrier as the \emph{cognitive synthesizability bottleneck}. Figure~\ref{fig:failure} illustrates this failure mode concretely. The teacher is being asked to do two things at once: search an implicit graph of possible evidence connections while composing fluent language. Each task individually is manageable, but together they exceed the model's reliable fusion capacity. \citet{dziri2023faith} show that auto-regressive Transformers tend to linearize multi-step compositional tasks into surface-level pattern matching, which makes joint graph search and text generation a particularly fragile combination. The underlying evidence paths themselves are often valid, so the bottleneck is mechanical rather than epistemic.

This observation suggests a straightforward solution: separate topology discovery from language generation. We propose \textbf{Graph-Constrained Semantic Chain Synthesis (GCSCS)}, a four-stage offline pipeline that operationalizes this decoupling. Instead of asking the teacher to reason about evidence relationships on the fly, GCSCS first atomizes each contract into standalone keyword-grounded QA facts. It then constructs contextual keyword centroids that capture each concept's semantic footprint, and searches a FAISS index over these centroids for non-looping, topologically admissible paths under five geometric constraints. The result is a pre-validated linear reasoning chain with verified cross-clause connectivity. Only then is the teacher invoked to compose a question and answer from this ready-made sequence.

\WFclear
\fussy
\endgroup
\makeatletter
\@ifpackageloaded{lineno}{\linenumbers}{}
\makeatother

Since the teacher is never asked to do more than read a pre-assembled evidence path and write a QA example around it, GCSCS replaces an overloaded joint operation with a single-pass articulation task. As a result, the quality-gate pass rate rises from $22.0\%$ to $94.5\%$, expanding the usable corpus 4.4 times to 80K examples \cite{hendrycks2021cuad}. Fine-tuning Qwen3-32B on this corpus improves closed-book Token F1 from $21.66\%$ to $38.58\%$ under a strict source-contract split \cite{hendrycks2021cuad}. An 18K matched-size ablation confirms that this gain tracks corpus scale rather than per-chain quality, establishing data yield as the primary bottleneck \cite{hendrycks2021cuad}. This is consistent with the data-quality literature showing that filter-defined yield matters as much as nominal generation volume \cite{zhou2023lima,chen2024alpagasus,chen2024dog}.

We further characterize the nature of this gain through zero-shot transfer to human-annotated benchmarks. On HotpotQA and MuSiQue, the improvement is statistically positive but small (approximately $+0.56\%$ on HotpotQA), indicating that the 80K gain primarily reflects legal answer-style adaptation rather than generalizable multi-hop reasoning, which bounds the scope of the contribution \cite{yang2018hotpotqa}. Additionally, We find that the model learns to paste citation IDs at the end of every answer as a format habit, not because it has located the relevant evidence in the source contract. When we inject retrieved context at inference time, this citation habit disappears entirely (CFR drops from $89\%$ to $8\%$). We call this the \emph{citation-as-suffix} behavior.

Our contributions are threefold:
\begin{itemize}

\item We formalize the \emph{cognitive synthesizability bottleneck},
a mechanical failure mode in which jointly discovering evidence
structure and verbalizing it over specialized corpora degrades data
yield regardless of model capacity, establishing data yield as the
primary scaling bottleneck in synthetic multi-hop~data~construction
for specialized domains.

\item We propose GCSCS, which resolves this bottleneck by separating
path enumeration from verbalization so that each operation stays
within the teacher's competence. A matched-size ablation establishes
that the gain comes from expanded data yield rather than per-chain
quality: constrained and unconstrained chains produce
indistinguishable downstream performance at equal~training~scale,
ruling out chain quality as a confound.

\item We show that fine-tuning on the resulting corpus improves
closed-book Token F1 from $21.66\%$ to $38.58\%$ under a strict
held-out split, and identify a \emph{citation-as-suffix} behavior
where the model appends citation tokens as a format habit rather
than grounding them in evidence, collapsing entirely when retrieved
context is injected at inference~time.

\end{itemize}

\section{Related Work}
\label{sec:related}

\subsection{Synthetic Data for Instruction Tuning}
Self-Instruct and Unnatural Instructions showed that teacher-generated instruction-response pairs can improve alignment \cite{wang2023selfinstruct,honovich2023unnatural}.
WizardLM increased instruction complexity through evolutionary rewriting \cite{xu2024wizardlm}.
RLHF-style and design-guided synthesis broadened the use of synthetic supervision \cite{ouyang2022training,guo2025synthetic,liu2026designer,zhu2026CHIMERA}.
These pipelines leave the evidence path implicit; we select it before teacher fusion.

\subsection{Multi-Hop Reasoning, Prompting, and Graph Supervision}
We use HotpotQA and MuSiQue as out-of-domain transfer benchmarks in Appendix~\ref{app:external}; WebQuestions, IIRC, and TriviaQA are related open-domain benchmarks included for context \cite{yang2018hotpotqa,trivedi2022musique,talmor2018web,ferguson2020iirc,joshi2017triviaqa}.
Inference-time methods (Chain-of-Thought, Tree-of-Thought, self-consistency, least-to-most, decomposed prompting, program-of-thought) improve reasoning without changing training data \cite{wei2022chain,yao2023tree,wang2023self,zhou2023least,khot2023decomposed,chen2023program}.
Our work changes the offline data instead.

\textbf{Path-before-generation graph synthesis.}
The paradigm of selecting evidence paths before generation is not new.
QA-GNN and GreaseLM show that graph structure can support question answering at inference time \cite{yasunaga2021qa,zhang2022greaselm}.
FactCG extracts context graphs from text for multi-hop fact checking and treats path enumeration as an upstream substrate for natural-language verification \cite{lei2025factcg}.
AIM-SciQA \cite{aimsciqa2026} (LREC~2026) constructs inter-document scientific QA via single-hop atomic extraction and semantic bridging across papers.
SciNets studies graph-constrained synthesis over scientific literature \cite{sciNets2026}.
Recent work additionally argues that synthetic multi-hop traces can improve multi-hop reasoning \cite{kabra2026learning}.
All four lines (a) treat evidence-path selection as an offline graph-search problem, (b) use a teacher LLM to verbalize a validated path into natural language, and (c) report a yield or quality benefit from the constraint.
Our methodological differences (contextual centroids in place of entity nodes, five-constraint admissibility, maximal-path pruning) are described in Section~\ref{sec:method}; domain generalization and threshold sensitivity are discussed in Appendix~\ref{app:limitations}.

LIMA shows that a small carefully-curated SFT set can match much larger noisy ones \cite{zhou2023lima}, AlpaGasus filters Self-Instruct-style data to improve downstream quality \cite{chen2024alpagasus}, and DoG argues for filtering-aware curation \cite{chen2024dog}.
Our observation sits in this space: on a high-boilerplate domain like CUAD, graph-constrained path search lets a teacher produce $4.4\times$ more samples that pass the same quality gate, and the downstream gain tracks corpus size rather than per-example chain quality.

\subsection{Retrieval, FAISS, and Citation Grounding}
Retrieval-Augmented Generation and dense passage retrieval retrieve evidence at inference time \cite{lewis2020rag,karpukhin2020dpr}, often using Sentence-BERT or SimCSE \cite{reimers2019sentence,gao2021simcse}.
Self-RAG and Active RAG add reflection or dynamic retrieval control \cite{asai2024selfrag,jiang2023active}, and hop-aware studies show that multi-step agentic retrieval remains brittle \cite{you2026agenticrag}.
FAISS \cite{johnson2019faiss} is commonly an inference-time vector index; we use it offline during data synthesis to propose conceptual edges.
Section~\ref{subsec:openbook} shows that a model trained in closed-book format loses citation behavior when retrieved context is injected at inference time.

\section{Methodology}
\label{sec:method}

The GCSCS pipeline (Figure~\ref{fig:pipeline}) has four stages: ontological atomization, contextual keyword centroid construction, constrained semantic chain building, and teacher fusion.
Appendix~\ref{app:method} provides implementation details on prompt, quality-gate failures, ID normalization, and threshold calibration.

\begin{figure}[tbp]
\centering
\includegraphics[width=\textwidth]{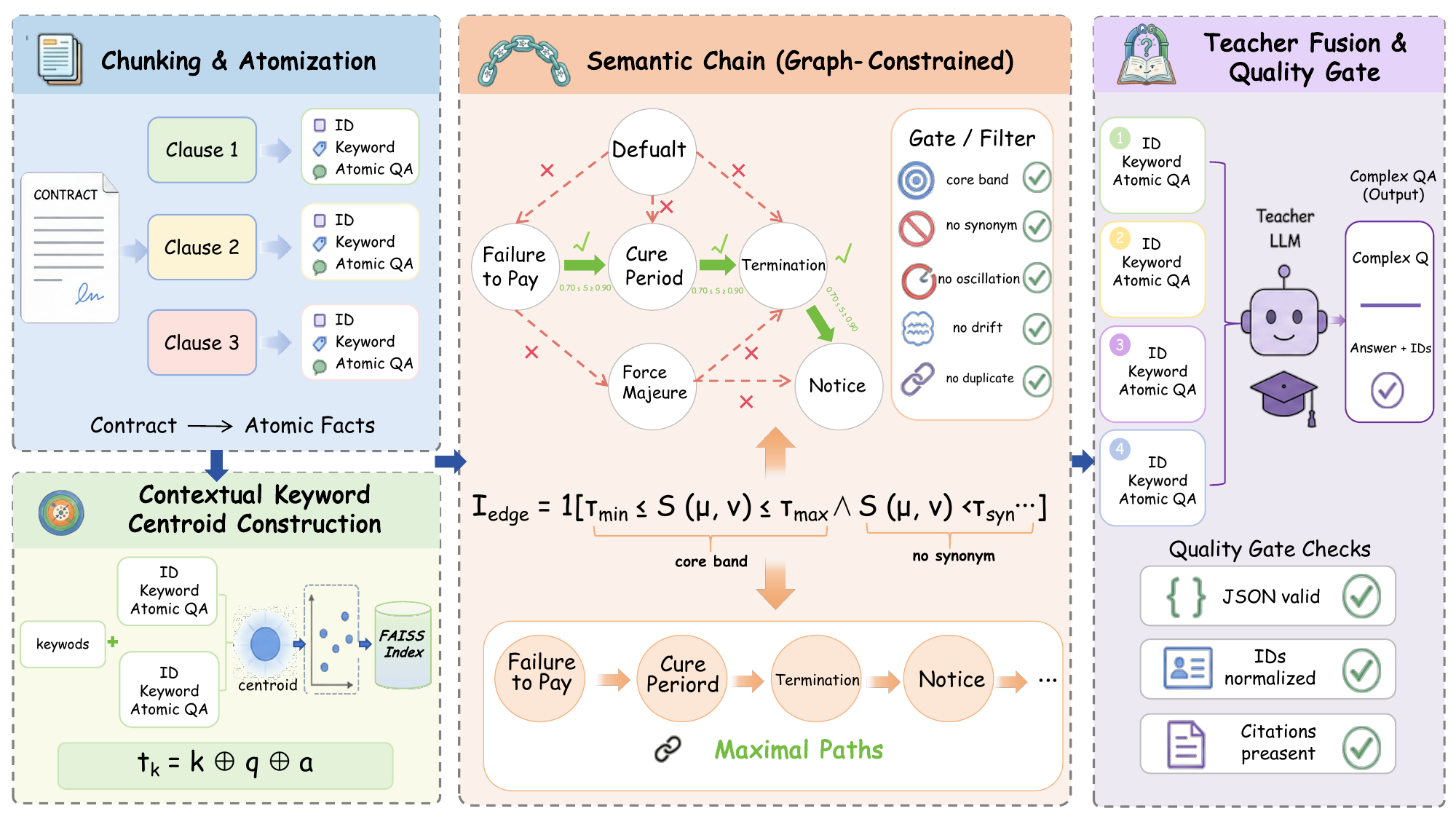}
\caption{Overview of the four-stage Graph-Constrained Semantic Chain Synthesis pipeline.}
\label{fig:pipeline}
\end{figure}

\subsection{Ontological Chunking and Atomization}
Given a contract $T$, the chunker identifies native legal boundaries (Article, Section, Clause, Paragraph, Schedule, Exhibit).
These boundaries reduce mid-clause truncation and preserve the local semantics of obligations, exceptions, and definitions.
If a document lacks explicit headings, sentence-boundary-aware splitting with a maximum block size of 1{,}200 characters is used instead.

For each chunk $c_i$, an extraction LLM $\mathcal{M}_{\text{ext}}$ at temperature 0.1 outputs a structured object $\mathcal{A}_i=\{k_j,q_j,a_j\}_{j=1}^{N_i}$ containing keywords and atomic QA pairs.
The extraction prompt forbids relative referents such as ``the preceding clause,'' so each fact stands alone without needing its surrounding context during later path construction.
Validated facts receive globally unique evidence identifiers that support citation tracking during teacher fusion and supplementary overlap audits.

\subsection{Contextual Keyword Centroids}
Once we have atomic QA pairs, the next challenge is connecting related but distinct facts across the contract. Direct QA-to-QA similarity tends to retrieve near-duplicate clauses from the same clause family and create semantic loops within a single contract.
To avoid this, we represent each keyword not by its surface string alone but by a contextual centroid that incorporates representative QA content.
For keyword $k$ appearing in $M$ QA contexts, we concatenate the keyword with at most two representative QA pairs:
\begin{equation}
t_k = k \oplus q_{k,1} \oplus a_{k,1} \oplus \cdots \oplus q_{k,K} \oplus a_{k,K}, \quad K=\min(2,M).
\end{equation}
The centroid vector is obtained by embedding and L2-normalizing this grounded text:
\begin{equation}
\hat{k}=\operatorname{Normalize}(\phi(t_k)),\quad \|\hat{k}\|_2=1.
\end{equation}
We instantiate $\phi$ with \texttt{text-embedding-v4}, which produces
1536-dimensional dense vectors. We index all centroids with FAISS \texttt{IndexFlatIP}, enabling exact nearest-neighbour search over all centroids. Because vectors are L2-normalized, inner product corresponds exactly to cosine similarity.

\subsection{Constrained Semantic Chain Building}
Let $G=(V,E)$ be a directed graph over contextual keyword centroids and let $S(\cdot,\cdot)$ denote cosine similarity in the centroid space.
For a partial path $P=(v_1,\ldots,v_{\mathrm{prev}},u)$, a candidate next node $v$ is admissible if and only if it passes all five semantic constraints and the lexical deduplication check.
We write the constraints as a piecewise function of the current depth $|P|$:
\begin{equation}
\tau_{\mathrm{prev}}^{*}(|P|) \;=\;
\begin{cases}
\text{undefined}, & |P|=1,\\[2pt]
0.75, & |P|\in\{2,3\},\\[2pt]
0.70, & |P|\geq 4,
\end{cases}
\qquad
\tau_{\mathrm{drift}}^{*}(|P|) \;=\;
\begin{cases}
\text{undefined}, & |P|=1,\\[2pt]
0.50, & |P|\geq 2.
\end{cases}
\end{equation}
The edge admissibility indicator is
\begin{equation}
\label{eq:edge}
\begin{aligned}
\mathbf{I}_{\mathrm{edge}}(u,v;P)
=\mathbf{1}\!\big[
&\,\tau_{\min}\leq S(u,v)\leq\tau_{\max}
\;\land\; S(u,v)<\tau_{\mathrm{syn}}\\
&\,\land\; \big(\tau_{\mathrm{prev}}^{*}(|P|)\text{ undef.}\;\lor\; S(v,v_{\mathrm{prev}})<\tau_{\mathrm{prev}}^{*}(|P|)\big)\\
&\,\land\; \big(\tau_{\mathrm{drift}}^{*}(|P|)\text{ undef.}\;\lor\; S(v,v_1)\geq\tau_{\mathrm{drift}}^{*}(|P|)\big)\\
&\,\land\; \lambda(v,P)=0
\,\big],
\end{aligned}
\end{equation}
\noindent where $\mathbf{1}[\cdot]$ is the indicator function,
$v_{\mathrm{prev}}$ is the node immediately preceding $u$ in $P$, and $v_1$ is the chain's origin. The static thresholds $(\tau_{\min},\tau_{\max})=(0.70,0.90)$ define the admissible cosine similarity band for each hop, and $\tau_{\mathrm{syn}}=0.95$ excludes near-synonyms. The depth-dependent threshold $\tau_{\mathrm{prev}}^{*}(|P|)$ is undefined at $|P|=1$, equals $0.75$ for $|P|\in\{2,3\}$, and $0.70$ for $|P|\geq 4$; it rejects $v$ if $S(v,v_{\mathrm{prev}})$ meets or exceeds this cutoff, preventing back-and-forth oscillation between adjacent hops. The threshold $\tau_{\mathrm{drift}}^{*}(|P|)$ is undefined at $|P|=1$ and equals $0.50$ for $|P|\geq 2$; it rejects $v$ if its similarity to the chain origin $v_1$ falls below $0.50$, preventing global semantic drift across the full chain. Finally, $\lambda(v,P)$ is a binary flag that excludes $v$ when its keyword label is lexically near-duplicate to any node already present in~$P$.

\paragraph{Constraint calibration.}
The interval $[\tau_{\min},\tau_{\max}]=[0.70,0.90]$ is chosen to
target edges that connect related but conceptually distinct concepts.
Under standard STS annotation, pairs scoring $3.5$--$4.0$ out of $5$
are considered related but not synonymous \cite{reimers2019sentence};
these pairs cluster within cosine similarity $[0.70,0.85]$ under
transformer encoders. We set $\tau_{\max}$ below $\tau_{\mathrm{syn}}=0.95$ to create a three-zone partition (weak association / admissible inferential steps / near-synonyms). While calibrated on CUAD development data, these values align with dense retrieval sweet spots~\cite{karpukhin2020dpr}. Theorem~\ref{thm:drift-necessity} and Theorem~\ref{thm:upper-bound} provide formal geometric justifications; empirical grounding for $\tau_{\max}$ is in Appendix~\ref{app:proof_upper}.

The piecewise $\tau_{\mathrm{prev}}^{*}$ tightens with depth ($0.75$ for the first two hops, $0.70$ at depth $\geq 4$) to enforce stricter anti-oscillation.
Lexical deduplication $\lambda(v,P)$ flags near-duplicate keyword labels via exact match, substring containment, character-overlap $>0.80$, or \texttt{difflib} sequence ratio $>0.60$.
FAISS retrieves $K=100$ candidates per node; the path search is a bounded best-first DFS with $B=5$, $L_{\max}=8$, and maximal-path pruning (Appendix~\ref{app:method}).

\noindent\textit{Assumption (density).}
Throughout this section, $V$ is assumed sufficiently dense in
$\mathbb{S}^{d-1}$ so that for the GCSCS constraint parameters
$(\tau_{\min}, \tau_{\max})$, a valid successor always exists at
each path extension step and no search branch terminates~prematurely.

\begin{definition}[Locally Constrained Semantic Chain]
\label{def:local-chain}
Let $G=(V,E)$ be a semantic graph where each contextual keyword centroid $v \in V$ is an $L_2$-normalized vector on $\mathbb{S}^{d-1}$ ($d \geq 2$). A semantic chain $P = (v_1, v_2, \dots, v_L)$ of length $L$ is \emph{locally constrained} with threshold $\tau_{\min}$ if $S(v_i, v_{i+1}) \geq \tau_{\min}$ for all $i \in \{1, \dots, L-1\}$, enforcing a minimum similarity between every pair of consecutive nodes.
\end{definition}

\begin{theorem}[Insufficiency of Local Semantic Constraints]
\label{thm:drift-necessity}
Let $P$ be a locally constrained semantic chain with $\tau_{\min} = \tfrac{7}{10}$ and $L \geq 3$. The local hop constraints are strictly insufficient to guarantee a global semantic anchoring of $S(v_1, v_L) \geq \tau_{\textup{drift}} = \tfrac{1}{2}$. Consequently, an explicit global drift constraint is geometrically necessary to prevent unbounded endpoint deviation.
\end{theorem}

\begin{proof}[Proof Sketch]
We provide a constructive proof via Gram matrix realizability
(full derivations in Appendix~\ref{app:proof_drift}).
For $L=3$, we construct the $3\times 3$ Gram matrix of $(v_1, v_2, v_3)$ with adjacent similarities fixed at $S = \tfrac{7}{10}$ and unknown $x = S(v_1, v_3)$. Sylvester's criterion for positive semi-definiteness yields the realizable domain $x \in \bigl[-\tfrac{1}{50}, 1\bigr]$. Since $-\tfrac{1}{50} < \tfrac{1}{2}$, the configuration $S(v_1, v_3) = -\tfrac{1}{50}$ is geometrically valid yet violates the global target. For $L > 3$, the local constraint is purely Markovian and imposes no geometric restoring force toward $v_1$. A second Gram matrix argument shows that for any $v_k$ with $S(v_1,v_k) \leq \tfrac{1}{2}$, there always exists a valid $v_{k+1}$ satisfying the local constraint while maintaining $S(v_1, v_{k+1}) < \tfrac{1}{2}$; the key bound follows from $\tfrac{7s}{10} \leq 0.35 < \tfrac{1}{2}$ for~$s \leq \tfrac{1}{2}$.
\end{proof}

\begin{remark}
Over just two hops ($L=3$) with $\tau_{\min}=0.70$, the end-to-end
similarity can reach $-\tfrac{1}{50}$ (${\approx}91.1^{\circ}$ drift), demonstrating that global semantic anchoring is not implied by any finite set of local hop constraints, however tight those constraints are~set.
\end{remark}

\begin{theorem}[Necessity of the Upper Bound for Escaping Dense Semantic Cliques]
\label{thm:upper-bound}
Let $G=(V,E)$ be a semantic graph where each contextual keyword centroid $v \in V$ is an $L_2$-normalized vector on $\mathbb{S}^{d-1}$. Define a dense boilerplate clique as a spherical cap $C(c,\epsilon) = \{v \in \mathbb{S}^{d-1} : \arccos(S(v,c)) \le \epsilon\}$ with angular radius $\epsilon \in (0, \pi/2)$ around a centre $c$.

\textbf{(i) Insufficiency of the lower bound.} If the edge admissibility condition only requires $S(v_i,v_{i+1}) \ge \tau_{\min}$, then for any clique satisfying $\cos(2\epsilon) \ge \tau_{\min}$, every pair $(u,v) \in C(c,\epsilon)\times C(c,\epsilon)$ satisfies the condition, permitting paths of arbitrary length to remain within $C(c,\epsilon)$.

\textbf{(ii) Sufficiency of the upper bound.} Under the additional
constraint $S(v_i,v_{i+1}) \le \tau_{\max}$, if $\epsilon < \frac{1}{2}\arccos(\tau_{\max})$, then for any $v_i \in C(c,\epsilon)$, the admissibility condition eliminates every $u \in C(c,\epsilon)$ as a candidate, forcing $v_{i+1} \notin C(c,\epsilon)$ and guaranteeing escape from the dense~clique.
\end{theorem}

\begin{proof}[Proof Sketch]
Follows from the spherical triangle inequality $\theta(u,v)\le 2\epsilon$ and monotonicity of $\cos(\cdot)$; full derivation and empirical verification in Appendix~\ref{app:proof_upper}.
\end{proof}

\begin{remark}
The GCSCS pipeline sets $\tau_{\max}=0.90$, enforcing a minimum angular separation of $\arccos(0.90)\approx 25.84^{\circ}$ between consecutive hops. Theorem~\ref{thm:upper-bound} guarantees escape from any boilerplate clique with $\epsilon < 12.92^{\circ}$. Appendix~\ref{app:proof_upper} provides empirical verification that $87.8\%$ of the $5{,}313$ DBSCAN-identified boilerplate clusters in CUAD satisfy this condition; the remaining $12.2\%$ are covered by the $\tau_{\mathrm{syn}}=0.95$ filter. Together, Theorem~\ref{thm:drift-necessity} and Theorem~\ref{thm:upper-bound} provide a closed geometric justification for the $[\tau_{\min},\tau_{\max}]$ band with each bound addressing a distinct failure~mode.
\end{remark}

\subsection{Teacher Fusion and Quality Gate}
Each accepted path is serialized into evidence blocks.
Every supporting fact is prepended by its global ID, and each node contributes at most three evidence items to control context length.
A generation LLM $\mathcal{M}_{\text{gen}}$ at temperature 0.2 writes a complex question and answer citing the evidence IDs.
Citation IDs serve as traceability markers that link each answer back to its source evidence items \cite{min2023factscore,manakul2023selfcheckgpt,gekhman2023trueteacher}. Whether a cited ID actually entails the answer is not verified at this stage; that check is left for downstream evaluation.
The quality gate $G_{\mathrm{gate}}(\cdot)$ comprises structured JSON validation, up to three retries, and ID normalization.
For a candidate path set $\mathcal{P}$, synthesis yield is
\begin{equation}
\mathrm{Yield}(\mathcal{P})
=
\frac{1}{|\mathcal{P}|}
\sum_{P_i\in\mathcal{P}}
\mathbf{1}\!\left[
G_{\mathrm{gate}}(\mathcal{M}_{\mathrm{gen}}(P_i))=1
\right].
\end{equation}
All pass rates in Section~\ref{sec:ablation} are measured after this gate.

\section{Experimental Setup}
\label{sec:exp}

\textbf{Corpus and model.}
We use the Contract Understanding Atticus Dataset (CUAD) \cite{hendrycks2021cuad}, whose contracts contain repeated clause families, defined terms, cross-references, and exceptions.
These properties suit multi-hop synthesis but also create a leakage risk because similar clauses recur across contracts.
The base model is Qwen3-32B; the synthesis pipeline is instantiated with DeepSeek (DS) and Gemini (GE) teachers to test whether gains persist across teacher distributions.

\textbf{Fine-tuning protocol.}
All SFT runs use LoRA \cite{hu2022lora,dettmers2023qlora}.
Unless stated otherwise, the LoRA rank is 8, alpha is 16, dropout is 0.05, target modules are all linear modules, learning rate is $1\times10^{-5}$, global batch size is 64, maximum sequence length is 4096, and training lasts one epoch.
Existing adapter weights are discarded before each run.
Appendix~\ref{app:setup} reports compute and reproducibility details.

\textbf{Metrics.}
Token F1 measures lexical overlap with the gold answer.
Citation Format Rate (CFR) measures adherence to the required
citation syntax. Evidence Recall (ER, diagnostic) measures the fraction of gold evidence IDs that the model explicitly cites.
Token F1 can reward answer length and teacher style; neither CFR nor ER alone proves citation faithfulness or genuine evidence grounding.
Appendix~\ref{app:metrics} gives full definitions and recommended
audits.

\textbf{Source-contract split.}
We split the 510 source contracts into 357 train / 51 development
/ 102 test (seed 42) before chunking, so training and held-out
chains are built from entirely independent source documents.
This removes exact source-contract contamination; template-level
near-duplicates across CUAD contracts remain a known risk.
Appendix~\ref{app:split} reports the data funnel and an ID-gap sanity audit; the GE and DS test sets contain 21{,}186 and
19{,}238~examples.

\section{Results}
\label{sec:results}

\subsection{Closed-Book CUAD-Style QA}
\label{subsec:main}

Table~\ref{tab:main} reports the in-pipeline closed-book results under a source-contract split.
DS-SFT improves DS-test Token F1 from $18.99\%$ to $37.19\%$, GE-SFT improves GE-test Token F1 from $21.66\%$ to $38.58\%$, and CFR rises from $51$--$54\%$ to $86$--$89\%$.
ER stays below $3\%$ and does not improve with training.
This combination is informative: the student learns to produce the citation template (emit \texttt{ID\_xxx} tokens at the end of an answer) but does not learn the underlying gold evidence-ID set.
We expand on this observation in Section~\ref{subsec:openbook}, where we show that the learned citation behavior collapses entirely under open-book retrieved context.

\begin{table}[tbp]
\centering
\caption{Closed-book multi-hop QA performance. ER is evidence recall against gold evidence IDs and is diagnostic only; higher is better. CFR ($\uparrow$) measures citation-format adherence.}
\label{tab:main}
\begin{tabular}{llccc}
\toprule
Model & Test Set & Token F1 & ER ($\uparrow$, diag.) & CFR ($\uparrow$) \\
\midrule
Qwen3-32B-Base & DS & 18.99\% & 0.40\% & 50.54\% \\
Qwen3-32B-DS-SFT & DS & 37.19\% & 0.36\% & 85.99\% \\
Qwen3-32B-Base & GE & 21.66\% & 2.74\% & 53.98\% \\
Qwen3-32B-Gemini-SFT & GE & 38.58\% & 2.34\% & 89.04\% \\
\bottomrule
\end{tabular}
\end{table}

Stratifying by chain length (Appendix~\ref{app:chainlen}) shows the gain holds across $\leq$$3$ / $4$--$5$ / $\geq$$6$ hop buckets.
The slightly larger $\geq$$6$-hop improvement is consistent with the hypothesis that multi-hop supervision helps on longer chains, but it does not by itself establish that the model has learned compositional long-chain reasoning as opposed to longer answer templates.

\subsection{Matched-Size Ablation}
\label{sec:ablation}

Table~\ref{tab:ablation} presents the central ablation.
Random chain sampling passes the teacher quality gate at $22.0\%$ and saturates at $\sim$18K usable examples. Semantic chain construction passes at $94.5\%$. When the semantic corpus is subsampled to the same 18K scale, Matched Semantic-SFT obtains $22.66\%$ Token F1 and Random-SFT obtains $22.51\%$.
We bold these two rows because they carry the causal weight of the ablation: at matched scale, Random and Matched Semantic are indistinguishable. The $0.15\%$ gap is a single-seed result lacking confidence intervals; we treat it as an indicative tie.

Given the indicative tie at 18K, the $16.92\%$ gain at 80K comes
from corpus scale, not from intrinsically better chains.
The graph constraints make each path easier for the teacher to
verbalize, which is why the pass rate rises from $22.0\%$ to
$94.5\%$ and the usable corpus grows $4.4\times$ without changing
the underlying generation model. This is consistent with the data-quality literature showing that filter-defined yield matters as much as nominal generation volume~\cite{zhou2023lima,chen2024alpagasus,chen2024dog}.

\begin{table}[tbp]
\centering
\caption{Matched-size ablation on the GE test set. \textbf{Bold rows} are the 18K matched-scale comparison (single seed, no CI). $\Delta$pp is relative to Random-SFT at 18K. ER is diagnostic-only.}
\label{tab:ablation}
\resizebox{\textwidth}{!}{%
\begin{tabular}{lcccccc}
\toprule
Configuration & Train Size & Pass Rate & Token F1 & $\Delta$pp & ER ($\uparrow$, diag.) & CFR ($\uparrow$) \\
\midrule
Base (no SFT) & --- & --- & 21.66\% & $-0.85$ & 2.74\% & 53.98\% \\
\textbf{Random-SFT} & \textbf{18K} & \textbf{22.0\%} & \textbf{22.51\%} & \textbf{0.00} & \textbf{2.69\%} & \textbf{60.59\%} \\
\textbf{Matched Semantic-SFT} & \textbf{18K} & \textbf{94.5\%} & \textbf{22.66\%} & \textbf{$+0.15^{*}$} & \textbf{2.70\%} & \textbf{63.55\%} \\
Semantic-SFT & 80K & 94.5\% & 38.58\% & $+16.07$ & 2.34\% & 89.04\% \\
\bottomrule
\end{tabular}}
\\[2pt]
{\footnotesize $^{*}$Single training seed; no multi-seed std or bootstrap CI computed. Treated as an indicative tie, not a directional claim.}
\end{table}

\FloatBarrier
\subsection{Data Volume Scaling}
\label{sec:scaling}

Figure~\ref{fig:scaling_combined} varies only the number of semantic-chain training examples on the GE closed-book test set.

\begin{figure}[htbp]
    \centering
    \includegraphics[width=\linewidth,height=0.50\textheight,keepaspectratio]{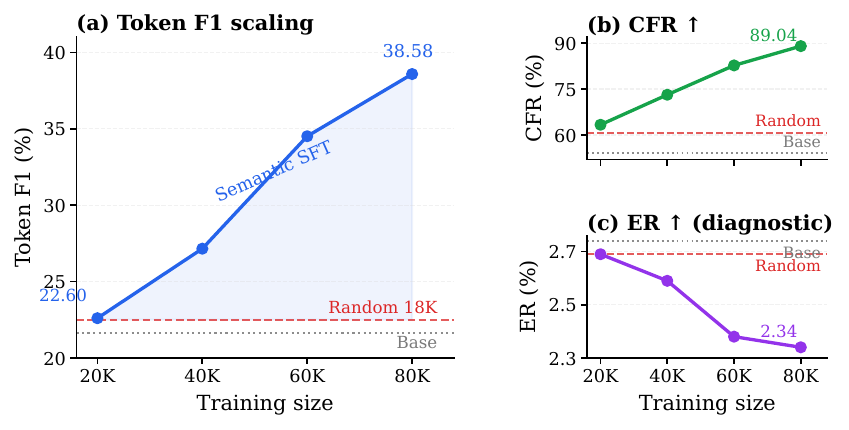}
    \caption{Data scaling analysis on the GE closed-book test set. Left: Token F1 vs.\ semantic-chain SFT training size, with the Random-SFT 18K and Base model as references. Right: citation format rate (CFR, higher is better) and evidence recall (ER, diagnostic, higher is better).}
    \label{fig:scaling_combined}
\end{figure}

Token F1 improves monotonically from $22.60\%$ at 20K to $38.58\%$ at 80K.
CFR rises from $63.33\%$ to $89.04\%$, while ER slightly decreases from $2.69\%$ to $2.34\%$.
The 20K semantic result ($22.60\%$) is nearly identical to the 18K random result ($22.51\%$), confirming that the mechanism is scale rather than per-example quality.
The monotonic rise in CFR alongside flat (or declining) ER is a diagnostic pattern: scaling teaches the model to produce better-formatted answers with citation tokens in the correct syntactic positions, but does not teach it to select the correct evidence IDs.
This separation between format learning and evidence learning persists across the entire 20K--80K range.

\subsection{Cross-Teacher Evaluation}
\label{subsec:cross}

Table~\ref{tab:cross} evaluates each SFT model on the other
teacher's held-out test set. GE-SFT reaches $39.28\%$ on DS-test and DS-SFT reaches $38.27\%$ on GE-test, preserving most in-domain gains across teacher styles and weakening a single-teacher mimicry explanation for the observed improvements.

\begingroup
\setlength{\intextsep}{10pt plus 2pt minus 2pt}
\begin{table}[tbp]
\centering
\caption{Cross-teacher generalization analysis.}
\label{tab:cross}
\begin{tabular}{lcc}
\toprule
Condition & Token F1 & vs Base $\Delta$pp \\
\midrule
Base on DS-test & 18.99\% & --- \\
DS-SFT on DS-test (in-domain) & 37.19\% & +18.20 \\
GE-SFT on DS-test (cross-A) & 39.28\% & +20.29 \\
Base on GE-test & 21.66\% & --- \\
GE-SFT on GE-test (in-domain) & 38.58\% & +16.92 \\
DS-SFT on GE-test (cross-B) & 38.27\% & +16.61 \\
\bottomrule
\end{tabular}
\end{table}
\endgroup

\subsection{Open-Book Distribution Shift}
\label{subsec:openbook}

We test whether the closed-book citation behavior survives a
distribution shift by injecting dense retrieved CUAD clauses into
the inference-time prompt while keeping weights and questions identical. CFR collapses from $85.99$--$89.04\%$ to $3.42$--$8.45\%$
(Table~\ref{tab:openbook_shift}), ER moves by less than $0.2\%$,
and Token F1 falls by $10$--$13\%$. Prompt-slot placement (user vs.\ system) changes nothing material. We interpret this as \emph{citation-as-suffix} behavior: at training time the answer always ends with \texttt{ID\_xxx} tokens, so the student raises their probability under the closed-book distribution. When a long retrieved passage is added to the prompt, the input distribution shifts enough that the model stops emitting citation IDs altogether, even though its answer quality does not drop proportionally. True grounding should remain stable or improve when evidence appears in context, but CFR collapses instead, bounding deployment value for legal RAG~applications.

\begin{table}[tbp]
\centering
\caption{Open-book distribution shift. Same weights and questions; only a long retrieved-context block is added. CFR collapse with ER unchanged is the citation-as-suffix signature.}
\label{tab:openbook_shift}
\begin{tabular}{llccc}
\toprule
Model & Setting & Token F1 & ER ($\uparrow$, diag.) & CFR ($\uparrow$) \\
\midrule
DS-SFT & Closed-book & 37.19\% & 0.36\% & 85.99\% \\
DS-SFT & Open-book (user) & 23.75\% & 0.34\% & 3.42\% \\
DS-SFT & Open-book (system) & 23.92\% & 0.35\% & 3.70\% \\
GE-SFT & Closed-book & 38.58\% & 2.34\% & 89.04\% \\
GE-SFT & Open-book (user) & 28.35\% & 2.50\% & 8.25\% \\
GE-SFT & Open-book (system) & 28.50\% & 2.54\% & 8.45\% \\
\bottomrule
\end{tabular}
\end{table}

We test whether mixed-format training can restore citation reliability. Augmenting the 80K closed-book corpus with $20$K open-book samples ($5\!:\!3\!:\!2$ Clean / Partial-Noise / Incomplete) raises open-book Token F1 to $45.55\%$ and closed-book Token F1 to $42.72\%$, but open-book CFR only recovers to $20.57\%$, far below the $95.62\%$ closed-book CFR of the same model (Table~\ref{tab:mixed}). This confirms that switching citation behavior from suffix-style formatting to genuine evidence grounding requires a fundamentally different training signal (see Appendix~\ref{app:limitations}).

\begin{table}[tbp]
\centering
\caption{Mixed-format SFT does not restore citation reliability. The 80K closed-book corpus is augmented with 20K open-book samples (Clean:Partial-Noise:Incomplete $=5\!:\!3\!:\!2$). GE-SFT reference rows are reported in Table~\ref{tab:openbook_shift}. Open-book CFR rises from $8.25\%$ to $20.57\%$, far below the $95.62\%$ closed-book CFR and below any threshold suitable for legal RAG deployment.}
\label{tab:mixed}
\begin{tabular}{llccc}
\toprule
Model & Setting & Token F1 & ER ($\uparrow$, diag.) & CFR ($\uparrow$) \\
\midrule
Mixed-SFT  & Closed-book    & 42.72\% & 2.41\% & 95.62\% \\
Mixed-SFT  & Open-book      & 45.55\% & 2.45\% & 20.57\% \\
\bottomrule
\end{tabular}
\end{table}

A LegalBench subset check (Appendix~\ref{app:legalbench}) shows GE-SFT and DS-SFT within $1\%$ of the Qwen3-32B-Base average ($91.71\%$ and $90.86\%$ vs.\ $91.90\%$), ruling out catastrophic forgetting on legal reasoning tasks in this subset.
Broad legal-reasoning transfer is not claimed.

\section{Discussion}
\label{sec:discussion}
Our findings indicate a data-yield mechanism rather than a per-chain quality mechanism: constrained path search improves teacher synthesizability, so downstream gains primarily come from having more usable training examples. A separate behavioral diagnosis is that citation formatting can improve without evidence grounding, indicating template adherence under the closed-book training schema rather than robust citation use. We therefore frame citation-as-suffix as an independent limitation signal rather than a marker of grounded reasoning quality.

The evidence is scoped to one base model (Qwen3-32B + LoRA), one teacher per axis, CUAD-tuned thresholds, and one fusion template; a high pass rate marks comfortable teacher fusion, not optimality of hop necessity. Cross-teacher results (Section~\ref{subsec:cross}) and external benchmarks (Appendix~\ref{app:external}) bound style adaptation; the system is an offline data-scaling pipeline for closed-book legal QA, not a legal assistant or citation-grounded RAG (Appendix~\ref{app:broader}).
Zero-shot transfer to HotpotQA and MuSiQue confirms gains do not generalize beyond CUAD (Appendix~\ref{app:external}; $+0.56\%$ on HotpotQA).

\section{Conclusion}
\label{sec:conclusion}

GCSCS decouples path selection from text generation and raises the CUAD teacher pass rate from $22.0\%$ to $94.5\%$, expanding the usable corpus $4.4\times$ to $80$K examples and improving closed-book Token F1 by ${\sim}16$--$18\%$. A matched-size ablation confirms the gain reflects corpus scale rather than per-chain quality. External transfer to human-annotated benchmarks is small ($+0.56\%$ on HotpotQA; Appendix~\ref{app:external}), diagnostic evidence recall stays flat across the scaling sweep, and citation format collapses under retrieved context. Together these diagnostics bound the contribution to closed-book legal-style and citation-template adaptation rather than general multi-hop reasoning;
broader domains, retrieval backends, and multi-seed validation remain future work. For reproducibility, the synthesis codebase and constructed 80K dataset will be open-sourced upon publication.

\bibliographystyle{plainnat}
\bibliography{references}

\appendix

\section{Proof of Theorem~\ref{thm:drift-necessity}}
\label{app:proof_drift}

We provide a constructive proof for each $L \geq 3$.

\textbf{Base case} ($L=3$). Consider the $3\times 3$ Gram matrix of
$(v_1, v_2, v_3)$:
\[
\mathbf{G} = \begin{pmatrix}
1 & \tfrac{7}{10} & x \\[4pt]
\tfrac{7}{10} & 1 & \tfrac{7}{10} \\[4pt]
x & \tfrac{7}{10} & 1
\end{pmatrix}, \qquad x = S(v_1, v_3).
\]
The configuration is realizable in $\mathbb{R}^d$ ($d \geq 2$) if and only if
$\mathbf{G} \succeq 0$. By Sylvester's criterion, the $1\times 1$ minor
equals $1>0$ and the $2\times 2$ minor equals $\tfrac{51}{100}>0$, both
independent of $x$; the sole binding condition is therefore
$\det(\mathbf{G}) \geq 0$. Computing via cofactor expansion:
\[
\det(\mathbf{G}) = -x^2 + \tfrac{49}{50}x + \tfrac{1}{50}
= -\tfrac{1}{50}(x-1)(50x+1).
\]
Setting $\det(\mathbf{G}) \geq 0$ gives the realizable domain
$x \in \bigl[-\tfrac{1}{50}, 1\bigr]$.
Since $-\tfrac{1}{50} < \tfrac{1}{2}$, the configuration
$S(v_1,v_3) = -\tfrac{1}{50}$ is geometrically valid
(the Gram matrix has rank~$2$, confirming realizability for all $d\geq 2$)
while violating the global target, despite every adjacent pair satisfying
$S(v_i,v_{i+1}) = \tfrac{7}{10}$.

\textbf{Extension to all $L > 3$}.
The local hop constraint $S(v_k, v_{k+1}) \geq \tfrac{7}{10}$
references only adjacent nodes and imposes no constraint coupling $v_{k+1}$
to $v_1$.
We show that for any node $v_k$ with $S(v_1, v_k) = s \leq \tfrac{1}{2}$,
there exists a geometrically valid $v_{k+1}$ satisfying
$S(v_k, v_{k+1}) = \tfrac{7}{10}$ and $S(v_1, v_{k+1}) < \tfrac{1}{2}$.

Consider the Gram matrix of the triple $(v_1, v_k, v_{k+1})$:
\[
\mathbf{H}(y) = \begin{pmatrix}
1 & s & y \\[4pt] s & 1 & \tfrac{7}{10} \\[4pt] y & \tfrac{7}{10} & 1
\end{pmatrix}, \qquad y = S(v_1, v_{k+1}).
\]
Computing $\det(\mathbf{H}) = -y^2 + \tfrac{7s}{5}y + \tfrac{51}{100}(1-s^2)$,
the realizable domain for $y$ is:
\[
y \in \left[\tfrac{7s}{10} - \tfrac{\sqrt{51(1-s^2)}}{10},\;
            \tfrac{7s}{10} + \tfrac{\sqrt{51(1-s^2)}}{10}\right].
\]
For any $s \leq \tfrac{1}{2}$, the linear term satisfies
$\tfrac{7s}{10} \leq \tfrac{7}{20} = 0.35$.
Since $\tfrac{\sqrt{51(1-s^2)}}{10} \geq 0$, the lower bound satisfies:
\[
y_{\min}(s) = \tfrac{7s}{10} - \tfrac{\sqrt{51(1-s^2)}}{10}
\leq \tfrac{7s}{10} \leq 0.35 < \tfrac{1}{2}.
\]
Thus, setting $y = y_{\min}(s)$ always yields a geometrically valid
$v_{k+1}$ that maintains the violation.
Therefore, for every $L \geq 3$ there exists a valid chain with
$S(v_1, v_L) < \tfrac{1}{2} = \tau_{\textup{drift}}$, and an explicit
global constraint $S(v_1, v_L) \geq \tau_{\textup{drift}}$ is a strict
geometric necessity. $\blacksquare$

\section{Proof of Theorem~\ref{thm:upper-bound} and Empirical Verification}
\label{app:proof_upper}

\subsection*{Proof}
Let $\theta(u,v)=\arccos(S(u,v))$ denote the geodesic distance on $\mathbb{S}^{d-1}$. Since $\epsilon\in(0,\pi/2)$, the spherical cap $C(c,\epsilon)$ is geodesically convex: any two points in the cap lie in an open hemisphere centred at $c$, and the length-minimising geodesic between them remains within the cap. The spherical triangle inequality therefore applies.

\textbf{Pairwise similarity lower bound inside the clique.} For any $u,v\in C(c,\epsilon)$:
\[  \theta(u,v)\;\le\;\theta(u,c)+\theta(c,v)\;\le\;2\epsilon.\]
Since $\cos(\cdot)$ is strictly decreasing on $[0,\pi]$:
\[  S(u,v)=\cos(\theta(u,v))\;\ge\;\cos(2\epsilon).\tag{$*$}\]

\textbf{Part~(i).} If $\cos(2\epsilon)\ge\tau_{\min}$, then by~$(*)$ every pair $(u,v)\in C(c,\epsilon)\times C(c,\epsilon)$ satisfies $S(u,v)\ge\tau_{\min}$. A search constrained only from below can traverse an arbitrarily long chain without leaving $C(c,\epsilon)$.

\textbf{Part~(ii).} Suppose $\epsilon<\frac{1}{2}\arccos(\tau_{\max})$, i.e. $2\epsilon<\arccos(\tau_{\max})$. Since $\cos(\cdot)$ is strictly decreasing: $\cos(2\epsilon)>\tau_{\max}$. By~$(*)$, for any $v_i\in C(c,\epsilon)$ and any $u\in C(c,\epsilon)$:
\[  S(v_i,u)\;\ge\;\cos(2\epsilon)\;>\;\tau_{\max}.\]
The upper-bound constraint $S(v_i,v_{i+1})\le\tau_{\max}$ therefore eliminates every $u\in C(c,\epsilon)$. The algorithm must select $v_{i+1}$ with $\theta(v_i,v_{i+1})\ge\arccos(\tau_{\max})>2\epsilon$, placing $v_{i+1}$ strictly outside $C(c,\epsilon)$. \hfill$\blacksquare$

\subsection*{Empirical Verification of Boilerplate Clique Radii}
To verify the geometric premise of Theorem~\ref{thm:upper-bound}, we applied DBSCAN ($\texttt{eps}=0.05$ cosine distance, $\texttt{min\_samples}=3$) to the $45{,}551$ $L_2$-normalized keyword centroids extracted from CUAD training contracts.\footnote{%
  $850$ centroids were excluded due to API retry failures during the re-embedding run; the analysis covers $98.2\%$ of the $46{,}401$ centroids reported in Table~\ref{tab:data_funnel}.}
The threshold $\texttt{eps}=0.05$ corresponds to a similarity floor of $0.95$, deliberately aligned with the $\tau_{\mathrm{syn}}=0.95$ synonym-exclusion boundary in Eq.~\eqref{eq:edge}.

DBSCAN identified $5{,}313$ dense boilerplate clusters. Of the $45{,}551$ centroids, $55.0\%$ were classified as noise points (i.e., isolated keywords not belonging to any dense cluster), confirming that the majority of legal keywords occupy unique conceptual positions in the embedding space and do not form repetitive template families. For each cluster, we computed the angular radius $\epsilon$ as the maximum angular distance from any member to the $L_2$-normalized cluster centroid. Table~\ref{tab:clique_radii} summarises the distribution.

\begin{table}[ht]
\centering
\caption{Angular radius distribution of CUAD boilerplate clusters. The theoretical bound is $12.92^{\circ}$.}
\label{tab:clique_radii}
\begin{tabular}{lc}
\toprule
Statistic & Value \\
\midrule
Mean $\bar{\epsilon}$  & $10.87^{\circ}$ \\
Median                 & $10.89^{\circ}$ \\
$p_{90}$               & $13.11^{\circ}$ \\
$p_{95}$               & $13.69^{\circ}$ \\
$p_{99}$               & $14.86^{\circ}$ \\
Max                    & $18.97^{\circ}$ \\
\midrule
Clusters with $\epsilon < 12.92^{\circ}$ & $4{,}667/5{,}313$ ($87.8\%$) \\
\bottomrule
\end{tabular}
\end{table}

\begin{figure}[htbp]
\centering
\includegraphics[width=\linewidth]{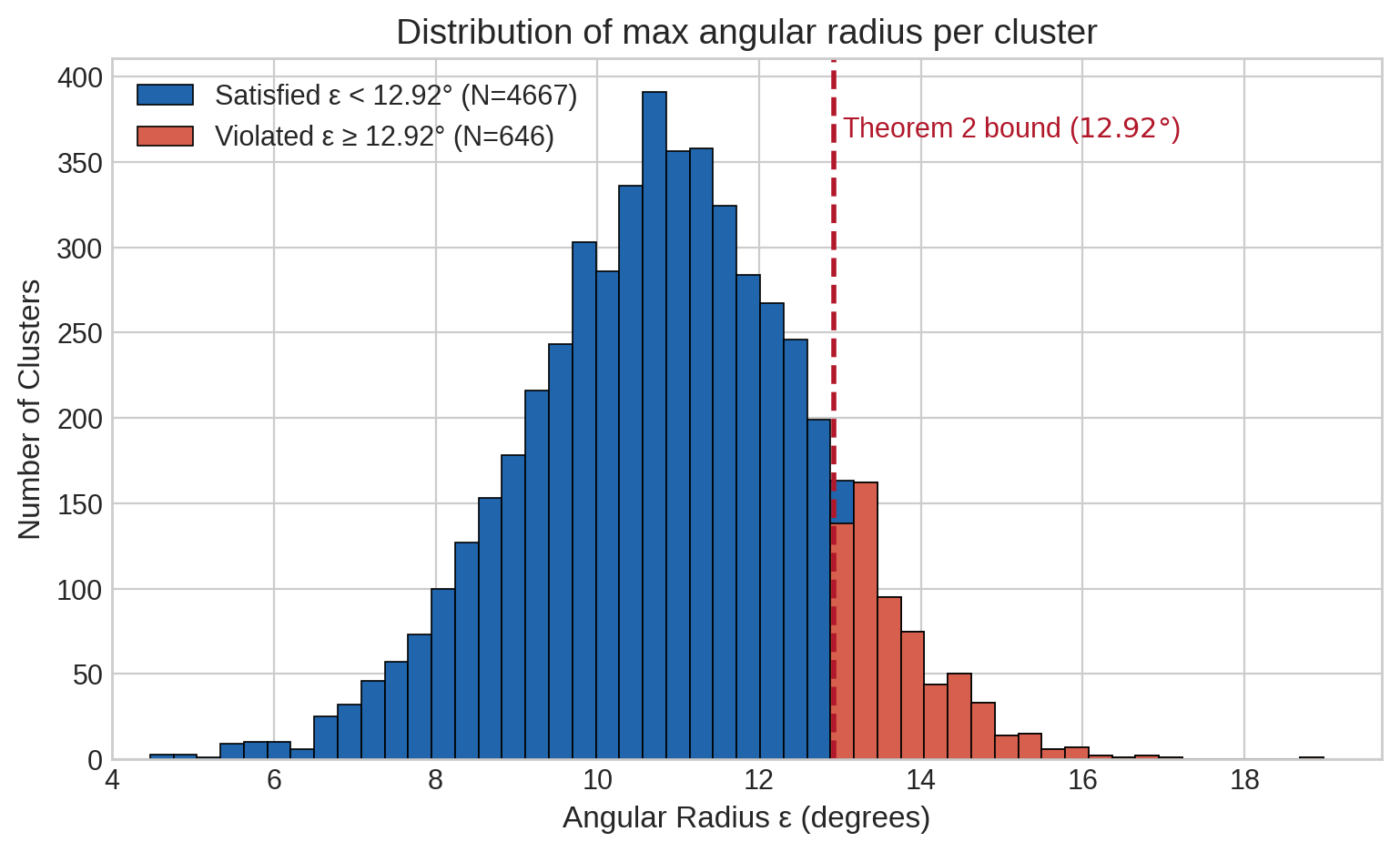}
\caption{Distribution of max angular radius $\epsilon$ per DBSCAN-identified
boilerplate cluster in CUAD ($n=5{,}313$ clusters).
Blue bars: clusters satisfying $\epsilon < 12.92^{\circ}$ ($N=4{,}667$,
$87.8\%$); red bars: clusters violating the bound ($N=646$, $12.2\%$).
The Theorem~\ref{thm:upper-bound} bound at $12.92^{\circ}$
($=\tfrac{1}{2}\arccos(0.90)$) is shown as a dashed red line.
The distribution peaks near $10$--$11^{\circ}$, confirming that the
large majority of CUAD boilerplate cliques fall within the geometric
escape guarantee of $\tau_{\max}=0.90$.}
\label{fig:clique_histogram}
\end{figure}

\begin{figure}[htbp]
\centering
\includegraphics[width=\linewidth]{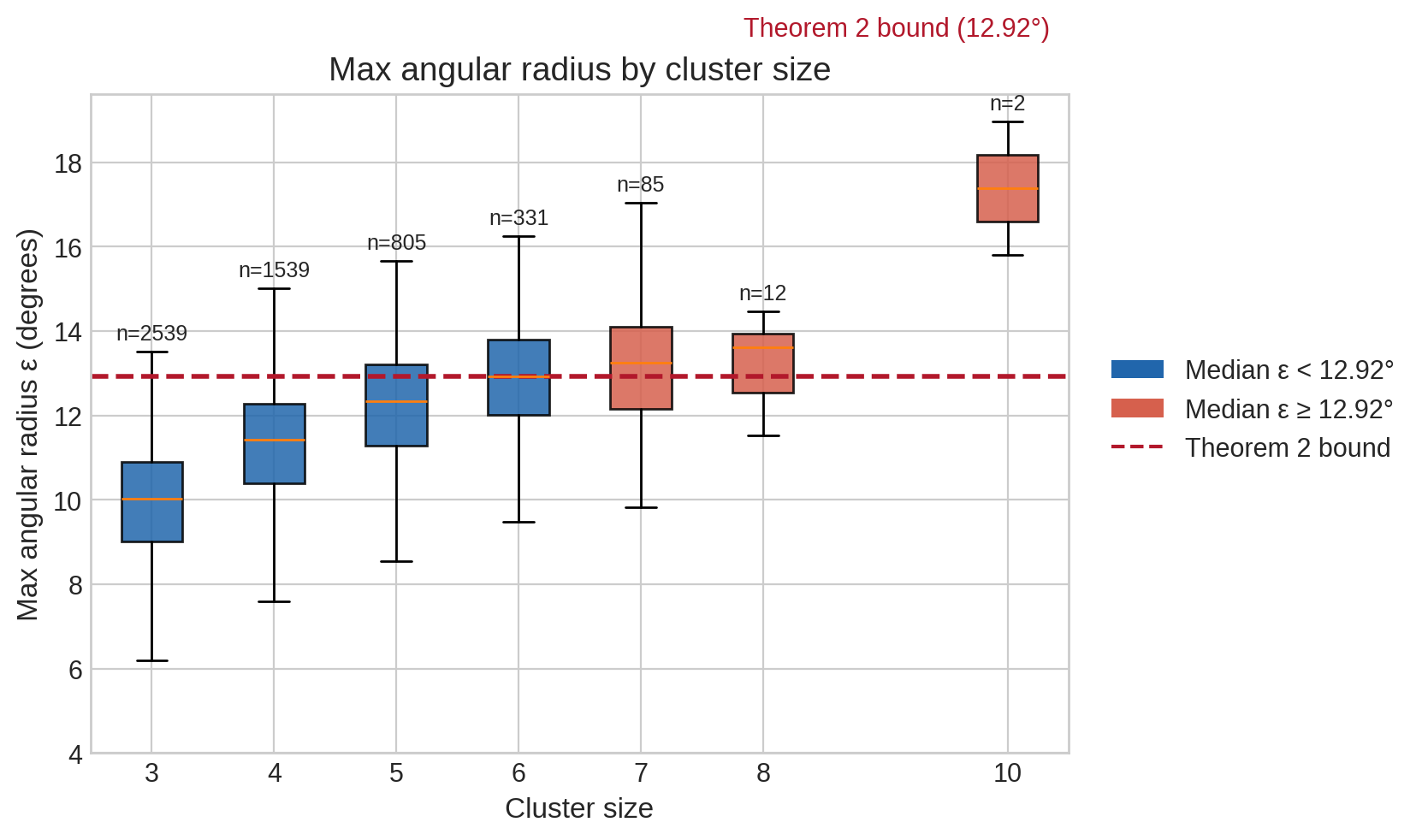}
\caption{Max angular radius $\epsilon$ by DBSCAN cluster size.
Each box shows the interquartile range; the orange line is the median;
whiskers extend to the 5th/95th percentile.
Blue boxes: median $\epsilon < 12.92^{\circ}$; red boxes: median
$\epsilon \geq 12.92^{\circ}$.
Sample counts $n$ are annotated above each box.
The absence of a systematic upward trend with cluster size confirms
that the Theorem~\ref{thm:upper-bound} geometric guarantee does not
degrade for larger boilerplate cliques within the observed size range.
The two clusters of size 10 ($n=2$) are included for completeness
but carry limited statistical weight.}
\label{fig:clique_scatter}
\end{figure}

For the $87.8\%$ majority, Theorem~\ref{thm:upper-bound} provides a strict geometric guarantee that $\tau_{\max}=0.90$ forces exit from the local cluster within one hop. For the remaining $12.2\%$ of clusters whose radius modestly exceeds the bound (up to $18.97^{\circ}$; $p_{99}=14.86^{\circ}$), two complementary constraints limit the practical impact: (1)~the synonym-exclusion filter $\tau_{\mathrm{syn}}=0.95$ independently eliminates near-identical nodes, and (2)~cluster sizes are small (median~$4$, max~$10$), making persistent local looping unlikely in practice. This dual-threshold design ensures forward chain momentum across the full CUAD embedding space.

\section{Additional Method Details}
\label{app:method}
We include implementation details for the graph-constrained synthesis algorithm here.

\textbf{Atomization and parsing.}
The extraction stage uses deterministic prompting with temperature 0.1.
The prompt requests a JSON object containing distinct keywords, atomic questions, and atomic answers.
The parser first attempts strict JSON decoding.
If decoding fails, an AST fallback parser and regular-expression sanitization repair common schema violations such as trailing commas, unescaped quotation marks, or malformed list delimiters.
Only validated atomic QA pairs enter the evidence database.

\textbf{Identifier discipline.}
Each atomic QA pair receives a globally unique, auto-incrementing integer identifier.
The identifier is independent of document boundaries.
This design makes generated answers traceable to source atoms.
It also enables the supplementary ID-gap audit in Appendix~\ref{app:split}.
That audit is not the primary split evidence.

\textbf{Path search details.}
For every unique keyword, the centroid text contains the keyword and up to two representative QA contexts.
All centroids are L2-normalized and indexed using FAISS \texttt{IndexFlatIP} \cite{johnson2019faiss}.
Candidate neighbors are first restricted to the top 100 nearest centroids.
They are then filtered by synonym exclusion, core association, predecessor compatibility, global drift, and lexical deduplication.
Lexical deduplication applies lowercased exact match, substring containment, character-overlap coefficient $>0.80$, and \texttt{difflib} sequence ratio $>0.60$ over visited keyword labels.
The search uses a bounded best-first depth-first traversal.
It caps branching at $B=5$ and grows paths to at most $L_{\max}=8$ nodes.
After enumeration, maximal-path pruning discards any path that is a strict prefix of a longer accepted path.
For the enumerated path set $\mathcal{P}$, the retained set is
\begin{equation}
\mathcal{P}_{\max}
=
\left\{
P\in\mathcal{P}:
\not\exists Q\in\mathcal{P}\ \mathrm{s.t.}\ P\prec Q
\right\},
\end{equation}
where $P\prec Q$ means that the ordered node sequence of $P$ is a strict prefix of $Q$.

\textbf{Teacher fusion details.}
The final path is serialized into evidence blocks.
Each node contributes at most three supporting evidence items.
The generation teacher $\mathcal{M}_{\text{gen}}$ runs at temperature 0.2.
It is prompted to write a complex question that requires all chain nodes and an answer that cites the corresponding evidence IDs.
The response schema requires a \texttt{complex\_question}, \texttt{complex\_answer}, and evidence registry.
The pipeline retries up to three times on API failure or schema violation.
Malformed IDs are normalized into a canonical \texttt{ID\_xxx} format when the target evidence item can be recovered from the source chain.
Otherwise the chain fails the quality gate.

\textbf{Path-enumeration pseudocode.}
Algorithm~\ref{alg:gcscs} compactly states the bounded constrained DFS used by Section~\ref{sec:method}.

\begin{algorithm}[ht]
\caption{Bounded constrained DFS for path enumeration (per seed node $v_1$).}
\label{alg:gcscs}
\begin{algorithmic}[1]
\REQUIRE Centroid set $V$, FAISS index, thresholds $(\tau_{\min},\tau_{\max},\tau_{\mathrm{syn}},\tau_{\mathrm{prev}}^{*}(\cdot),\tau_{\mathrm{drift}}^{*}(\cdot))$, $B$, $K$, $L_{\max}$.
\STATE $\mathcal{P}\leftarrow\emptyset$;\quad push partial path $P=(v_1)$ onto stack.
\WHILE{stack is not empty}
  \STATE Pop $P=(v_1,\ldots,u)$.
  \IF{$|P|\geq 2$} \STATE add $P$ to $\mathcal{P}$. \ENDIF
  \IF{$|P|=L_{\max}$} \STATE \textbf{continue}. \ENDIF
  \STATE $\mathcal{N}\leftarrow$ top-$K$ FAISS neighbors of $u$.
  \STATE $\mathcal{N}^{\mathrm{ok}}\leftarrow\{v\in\mathcal{N}\setminus P:\mathbf{I}_{\mathrm{edge}}(u,v;P)=1\}$ using Eq.~\eqref{eq:edge}.
  \STATE Sort $\mathcal{N}^{\mathrm{ok}}$ by $S(u,v)$ and keep the top $B$.
  \FORALL{$v\in\mathcal{N}^{\mathrm{ok}}$}
     \STATE push $P\cup(v)$.
  \ENDFOR
\ENDWHILE
\STATE $\mathcal{P}_{\max}\leftarrow\{P\in\mathcal{P}:\not\exists Q\in\mathcal{P}\text{ s.t.\ }P\prec Q\}$ \COMMENT{maximal-path pruning}
\RETURN $\mathcal{P}_{\max}$.
\end{algorithmic}
\end{algorithm}

\paragraph{Computational complexity.}
Algorithm~\ref{alg:gcscs} decomposes the offline synthesis cost into two
phases: dense retrieval and bounded DFS traversal.

\emph{Dense retrieval.}
All $|V|$ centroid vectors are indexed in a single FAISS \texttt{IndexFlatIP}
structure.
For each node $u$ expanded during DFS (line~7 of Algorithm~\ref{alg:gcscs}),
retrieving the top-$K$ nearest neighbours requires an exhaustive inner product
over the full index at cost $O(|V|\cdot d)$, where $d=1536$.
To avoid redundant computation across the $|V|$ seed-node DFS runs, each
node's top-$K$ neighbour list is precomputed once and cached before the
traversal phase; subsequent DFS expansions of the same node read from cache
at $O(1)$.
Under this single-pass retrieval strategy, the total cost of the dense
retrieval phase is $O(|V|^{2}\cdot d)$ in the worst case (all $|V|$ nodes
queried), and significantly less in practice because the stringent
$\mathbf{I}_{\mathrm{edge}}$ filter leaves many nodes unreachable.

\emph{Bounded DFS traversal.}
With neighbour lists precomputed, each DFS expansion reads $K=100$ pre-cached candidates and applies the admissibility filter to select at most $B=5$ successors.
The total number of states explored per seed node is bounded by
\begin{equation}
    |\text{States per seed}|
    \;\le\;
    \sum_{l=1}^{L_{\max}} B^{\,l}
    \;=\;
    \frac{B(B^{L_{\max}}-1)}{B-1}
    \;\approx\; 4.8\times10^{5}.
    \label{eq:dfs_states}
\end{equation}
Since $\sum_{l=1}^{L} B^{l} = \Theta(B^{L})$ (the geometric series is dominated by its final term), the traversal cost over all $|V|$ seed nodes is $O(|V|\cdot B^{L_{\max}})$. Because $B^{L_{\max}}$ and $d$ are constants independent of corpus size, both phases combine additively:
\begin{equation}
    \text{Total} \;=\;
    O\!\left(|V|^{2}\cdot d \;+\; |V|\cdot B^{L_{\max}}\right)
    \;=\; O\!\left(|V|^{2}\cdot d\right),
    \label{eq:total_complexity}
\end{equation}
where the retrieval phase dominates. This guarantees that the offline path-enumeration pipeline scales quadratically in the corpus size $|V|$, making it computationally tractable for corpus-scale synthesis.

\section{Training and Reproducibility Details}
\label{app:setup}
We report the training and compute details for Section~\ref{sec:exp}.

All fine-tuning experiments use Qwen3-32B with LoRA rank $r=8$, alpha 16, dropout 0.05, target modules set to \texttt{all-linear}, learning rate $1\times10^{-5}$, global batch size 64, maximum sequence length 4096, and one epoch.
Pre-existing adapter weights are discarded before each training run.
All reported SFT conditions use the same optimization settings unless they are explicitly labeled Mixed-SFT.
Experiments were conducted on NVIDIA H800 GPUs with 80GB memory.
Standard SFT runs used 8 H800 GPUs.
The Mixed-SFT run used 24 H800 GPUs.
The estimated total compute across reported training and inference experiments is approximately 2,400 GPU-hours.
The use of LoRA and QLoRA-style efficient fine-tuning follows prior parameter-efficient adaptation work \cite{hu2022lora,dettmers2023qlora}.

\section{Split and Overlap Analysis}
\label{app:split}
We record the source-contract split, the data construction funnel, and the supplementary overlap audit.

\begin{table}[tbp]
\centering
\small
\caption{Data construction funnel under the source-contract split. Final train QA denotes filtered semantic-chain QA before subsampling to the 80K SFT corpus. The held-out test uses 102 source contracts; 94 contract keys produce final QA after Step3/Step4 filtering.}
\label{tab:data_funnel}
\begin{tabular}{lrrrrrr}
\toprule
Split & Contracts & Chunks & Atomic QA & Embedded centroids & Semantic paths & Final QA \\
\midrule
Train & 357 & 16,931 & 69,654 & 46,401 & 85,499 & 80,799 \\
Test paths & 102 & 4,020 & 16,672 & 13,268 & 22,272 & --- \\
Test-GE & 94 & --- & --- & --- & --- & 21,186 \\
Test-DS & 94 & --- & --- & --- & --- & 19,238 \\
\bottomrule
\end{tabular}
\end{table}

The primary split is performed before chunking, atomization, path construction, and teacher fusion. 
Using seed 42, the 510 CUAD source contracts are split into 357 train, 51 development, and 102 test contracts. 
To strictly ensure data isolation, the held-out test set is constructed using an independent execution pipeline. 
Specifically, training paths are synthesized strictly from the atomized records of the training subset, while held-out paths are generated solely from the test subset. 
After path enumeration, the validated paths undergo teacher fusion, generating separate final evaluation datasets corresponding to the Gemini and DeepSeek teacher models. 
Only 94 of the 102 test contracts appear in the final QA composition because some contracts do not yield chains that survive Step 3 path construction and Step 4 teacher fusion.

\begin{table}[tbp]
\centering
\small
\caption{Held-out test set composition after teacher fusion. Hop buckets use the Step3 chain length written in metadata.}
\label{tab:test_composition}
\begin{tabular}{lrrrrr}
\toprule
Test set & Final QA & Avg. evidence IDs & 3-hop & 4--5-hop & $\geq$6-hop \\
\midrule
GE & 21,186 & 4.62 & 4,308 & 6,303 & 10,575 \\
DS & 19,238 & 5.58 & 4,183 & 5,793 & 9,262 \\
\bottomrule
\end{tabular}
\end{table}

We also estimate document-boundary overlap using jumps in globally unique evidence IDs.
A $p_{95}$ identifier-gap threshold of 15 is used to infer document boundaries.
This heuristic identifies 3 overlapping inferred document units and 303 overlapping QA pairs.
The estimated overlap does not increase across the 40K, 60K, and 80K training scales.
Because the main split is source-contract-level, this ID-gap result is only a secondary audit for progressive contamination.
It is not the basis of the split claim.

The remaining leakage risk is template-level rather than source-contract-level.
CUAD contracts can contain repeated boilerplate and near-duplicate legal clauses across different source files.
Future work should report exact duplicate, near-duplicate, answer-overlap, and evidence-overlap statistics across the source-contract split.

\section{Metric Definitions and Faithfulness Caveats}
\label{app:metrics}
We define the automatic metrics and state their limits.

\textbf{Token F1.}
Token F1 is computed between the model response and the gold synthetic answer after standard normalization.
It is useful for comparing models within a fixed answer style.
It can reward longer answers, repeated legal terminology, citation templates, and teacher-specific phrasing.
It does not identify whether the answer used all hops in the chain.

\textbf{Evidence Recall.}
Let $\mathrm{IDs}(y_i)$ denote the evidence IDs emitted by prediction $y_i$.
Let $\mathcal{E}_i^\star$ denote the gold evidence-ID set.
ER is the fraction of gold evidence IDs that appear in the model response:
\[
\mathrm{ER}
=
\frac{1}{N}\sum_{i=1}^{N}
\frac{|\mathrm{IDs}(y_i)\cap\mathcal{E}_i^\star|}
{|\mathcal{E}_i^\star|}.
\]
Higher is better, but ER is diagnostic only.
It measures exact recall of gold evidence identifiers.
It does not measure whether the cited evidence entails the answer.
The low ER values in the main tables show that citation-format learning is much stronger than exact evidence-ID recovery.

\textbf{Citation Format Rate.}
CFR measures whether the model follows the required citation syntax:
\[
\mathrm{CFR}
=
\frac{1}{N}\sum_{i=1}^{N}
\mathbf{1}\!\left[|\mathrm{IDs}(y_i)|>0\right].
\]
It does not measure whether the cited evidence entails the claimed reasoning step.
A model can achieve high CFR while citing an ID that is syntactically valid but semantically unsupported.
Future evaluation should add invalid-citation rate, evidence precision, citation-support entailment accuracy, human citation audits, and support ordering accuracy.

\section{Chain Topology Diagnostics}
\label{app:topology}
We report topology statistics for the semantic path builder.
These diagnostics describe path construction behavior.
They do not directly measure reasoning quality.

The train and held-out test path builders produce similar topology statistics.
The train split yields 85,499 paths with average chain length 5.66, average adjacent-hop similarity 0.7625, and average endpoint similarity 0.6094.
The held-out test split yields 22,272 paths with average chain length 5.57, average adjacent-hop similarity 0.7631, and average endpoint similarity 0.6141.
The chain-length distributions are also similar.
In train, lengths 3/4/5/6/7/8 account for 21.1/17.2/12.1/8.3/6.9/34.5\% of paths.
In test, the corresponding distribution is 21.9/18.1/11.9/8.7/7.1/32.1\%.
This stability shows that the held-out chain builder follows the same constrained topology regime as the training builder.
Human hop-necessity and citation-support validation remain future work.

\section{Chain-Length Bucket Stratification}
\label{app:chainlen}

Table~\ref{tab:chainlen} stratifies the in-pipeline closed-book results from Section~\ref{subsec:main} by Step3 chain length.
Gains are slightly larger in longer buckets, especially for DS-SFT at $\geq$$6$ hops.
This pattern is consistent with better use of multi-hop supervision but does not by itself prove improved long-chain reasoning, since longer buckets may differ in answer length, lexical overlap, and question difficulty.
Future work should add answer-length-normalized F1, evidence-ordering accuracy, and per-question hop-necessity audits.

\begin{table}[tbp]
\centering
\caption{Performance across chain-length buckets on the GE/DS test sets.}
\label{tab:chainlen}
\begin{tabular}{lcccccc}
\toprule
Chain Length & Base-DS & DS-SFT & $\Delta$pp & Base-GE & GE-SFT & $\Delta$pp \\
\midrule
$\leq$ 3-hop & 18.59\% & 35.98\% & 17.39 & 21.61\% & 37.32\% & 15.71 \\
4--5-hop & 19.12\% & 36.64\% & 17.52 & 21.63\% & 38.13\% & 16.50 \\
$\geq$ 6-hop & 19.09\% & 38.08\% & 18.99 & 21.69\% & 39.35\% & 17.66 \\
\bottomrule
\end{tabular}
\end{table}

\section{LegalBench Anti-Forgetting Check}
\label{app:legalbench}

Domain-specific fine-tuning can damage general capabilities \cite{kirkpatrick2017overcoming}.
We evaluate a small LegalBench subset \cite{guha2023legalbench,wang2023maud} as a check that GCSCS-driven SFT does not catastrophically forget related contract-understanding tasks.
Table~\ref{tab:legalbench} shows that GE-SFT and DS-SFT stay within $1\%$ of the Qwen3-32B-Base average ($91.71\%$ and $90.86\%$ vs.\ $91.90\%$).
This rules out catastrophic forgetting on this subset; we do not interpret it as evidence of broad legal-reasoning transfer.

\begin{table}[tbp]
\centering
\caption{Zero-shot transfer and forgetting analysis on a LegalBench subset.}
\label{tab:legalbench}
\begin{tabular}{lccc}
\toprule
LegalBench Task & Base & DS-SFT & GE-SFT \\
\midrule
consumer\_contracts\_qa & 95.71\% & 95.71\% & 95.71\% \\
contract\_nli\_notice\_on\_compelled\_disclosure & 97.18\% & 95.77\% & 97.18\% \\
cuad\_audit\_rights & 88.16\% & 86.76\% & 87.75\% \\
cuad\_non-compete & 85.75\% & 84.62\% & 86.43\% \\
\midrule
Average & 91.90\% & 90.86\% & 91.71\% \\
\bottomrule
\end{tabular}
\end{table}

\section{Limitations and Future Work}
\label{app:limitations}

The experiments in this paper are scoped to a single domain (CUAD legal contracts), a single base model (Qwen3-32B with LoRA), and a single embedding model (\texttt{text-embedding-v4}). We discuss the main methodological boundaries below.

\subsection{Scope of the yield claim}

The yield and downstream Token F1 gains are measured under a fixed pipeline configuration. All five admissibility thresholds ($\tau_{\min}$, $\tau_{\max}$, $\tau_{\mathrm{syn}}$, $\tau_{\mathrm{prev}}^{*}$, $\tau_{\mathrm{drift}}^{*}$) were calibrated on the CUAD development set; their behaviour on other high-boilerplate domains (e.g., SEC filings, medical guidelines) has not been evaluated. Whether the $[0.70, 0.90]$ admissibility band transfers without re-tuning is an open question.

The matched-size comparison at 18K (Table~\ref{tab:ablation}) is based on a single training seed. While the $0.15\%$ gap is consistent with a tie, multi-seed confidence intervals and a paired bootstrap test would provide stronger statistical footing for this conclusion. The data scaling curve in Figure~\ref{fig:scaling_combined} uses a single run per scale point for the same reason.

\subsection{Baseline coverage}

The current experiments compare GCSCS to random chain sampling at matched and full scale. Comparisons to BM25- or entity-graph-based chain construction, within-document random chains, and established path-before-generation pipelines such as FactCG~\cite{lei2025factcg} would further characterize the specific contribution of dense semantic retrieval to the yield improvement. These comparisons are natural extensions of the present work.

\subsection{Embedding and model sensitivity}

All centroid vectors are computed with a single encoder family. Replacing \texttt{text-embedding-v4} with alternative encoders (e.g., Sentence-BERT~\cite{reimers2019sentence}, SimCSE~\cite{gao2021simcse}) would establish whether the admissibility band is encoder-specific. Similarly, the downstream SFT experiments use only Qwen3-32B; results on smaller or differently pre-trained base models are not reported.

\subsection{Reasoning and grounding evaluation}

Token F1 and CFR measure lexical overlap and format adherence, respectively; neither directly measures multi-hop reasoning depth. Hop-necessity audits (evaluating whether each hop in the chain is causally required for the answer), per-citation NLI scoring, and blinded human evaluation of chain quality remain important directions for validating the reasoning content of the synthesized data.

\subsection{Citation behaviour under distribution shift}

The citation-as-suffix collapse documented in Section~\ref{subsec:openbook} shows that closed-book citation training and open-book RAG inference are structurally incompatible under the current schema. Mixed-format training partially mitigates but does not resolve this incompatibility. A training approach that conditions citation emission on the presence of retrieved evidence, rather than treating it as a positional suffix, is a natural direction for future work.

\section{Broader Impacts, Licensing, and Deployment Risks}
\label{app:broader}

\textbf{Positive impact.}
Lower-cost construction of domain-specific multi-hop SFT supervision can help small legal-tech teams that cannot afford human-authored chain annotation.
The yield observation also clarifies a practical bottleneck for synthetic data work in any high-boilerplate, repeat-clause domain: it is teacher-acceptability of candidate chains, not nominal generation volume, that limits scale.

\textbf{Fabricated legal authority is a concrete failure mode of this exact system.}
The closed-book Mixed-SFT model emits $\sim$$95\%$ syntactically valid \texttt{ID\_xxx} citations while exact gold-evidence-ID recall (ER) sits at $\sim$$2.4\%$.
Mechanically, this means the model is essentially always emitting an evidence-ID-shaped string, while the string almost never matches any specific evidence the answer actually depends on.
For a non-expert reader looking at the closed-book output --- well-formatted answer ending in citation IDs --- this looks like an authoritative legal citation.
For a deployment context, this is the failure mode of \emph{fabricated legal authority}: the system produces a confident citation to a clause-ID that has no semantic relationship to the claim.
The open-book CFR collapse to $8\%$ shows that the model does not robustly cite evidence that is actually present in the prompt.
We therefore explicitly state that the closed-book SFT system in this paper is \textbf{not safe to deploy} in any setting where a non-expert user might treat its citation output as a genuine legal reference.

\textbf{Required deployment safeguards.}
Any legal-workflow use of a model trained with this kind of synthetic citation supervision should require, at minimum: (i) inference-time retrieval-grounded verification (e.g., per-citation NLI between the cited text and the answer); (ii) abstention when no retrieved span entails the answer; (iii) audit logs that surface invalid citations to a reviewer; (iv) expert human review before any output is shown to an end user; and (v) explicit user-facing labelling that the system is a research artefact and not a legal advisor.
We do not claim that adding these safeguards would suffice; we claim only that they are necessary.

\textbf{Dataset license.}
CUAD \cite{hendrycks2021cuad} is released under CC BY 4.0.
Our use --- atomization, embedding, path enumeration, and teacher-fused multi-hop QA generation --- falls within the dataset's permitted use for research and adaptation provided that attribution is preserved.
We do not redistribute CUAD source contracts.
We intend to release the synthesized derivative QA examples and the synthesis pipeline code; commercial redistribution of derivatives is not authorized by us.

\textbf{Teacher-LLM terms-of-service.}
The synthesis pipeline calls two commercial teacher LLMs (Gemini and DeepSeek) for atomization and fusion.
Both providers' terms of service should be reviewed before any public release of the resulting student weights.
In particular, distillation/training-on-output clauses vary across providers and across model versions and may restrict release of weights derived from teacher output.
We treat this as a release-time constraint on the LoRA adapters and on the synthesized SFT corpus, and we will publish the explicit ToS check at public release time.

\textbf{Energy and compute footprint.}
The reported training and inference experiments use approximately $2{,}400$ NVIDIA H800 GPU-hours total (Appendix~\ref{app:setup}).
A non-trivial fraction is spent on teacher inference for the synthesis pipeline rather than on student training.
Practitioners reproducing this work should plan for both budgets and may prefer cheaper open-weight teachers when the goal is yield study rather than absolute quality.

\textbf{Negative-use considerations.}
The same yield mechanism --- ``find candidate evidence paths that the teacher comfortably verbalizes'' --- can in principle be used to generate large volumes of plausible-looking but unverified legal text.
Combined with the citation-as-suffix behavior we document, this creates a non-trivial misinformation surface in legal contexts.
We recommend that downstream users do not deploy citation-emitting closed-book legal QA models without an independent grounding verifier.

\section{Human-Annotated Multi-Hop Transfer}
\label{app:external}

To further bound whether in-pipeline gains reflect generalizable multi-hop reasoning, we run zero-shot closed-book inference on $3{,}500$ bridge-type HotpotQA \cite{yang2018hotpotqa} and $1{,}500$ answerable MuSiQue \cite{trivedi2022musique} validation questions (no supporting documents; paired $t$-tests on per-sample Token F1). Table~\ref{tab:external} summarizes the outcome: Token F1 improvements are at most ${\sim}0.6\%$---orders of magnitude smaller than the $+16\%$--$+18\%$ gains on our CUAD-style test---so the headline improvement is dominated by domain-specific answer style and citation-template adaptation rather than transferable multi-hop reasoning. The small non-zero HotpotQA lift rules out a degenerate ``no transfer'' story while confirming that any transfer remains marginal.

\begingroup
\setlength{\intextsep}{10pt plus 2pt minus 2pt}
\begin{table}[H]
\centering
\small
\caption{Zero-shot closed-book transfer on human-annotated benchmarks. No fine-tuning on either benchmark was performed; $p$-values are paired $t$-tests on per-sample Token F1.}
\label{tab:external}
\begin{tabular*}{\linewidth}{@{\extracolsep{\fill}}lcccccc}
\toprule
Model & \multicolumn{3}{c}{HotpotQA} & \multicolumn{3}{c}{MuSiQue} \\
\cmidrule(lr){2-4} \cmidrule(lr){5-7}
 & F1 & EM & $p$ & F1 & EM & $p$ \\
\midrule
Qwen3-32B-Base   & 22.97\% & 12.43\% & --- & 13.14\% & 3.13\% & --- \\
Qwen3-32B-GE-SFT & 23.53\% & 12.97\% & 0.035 & 13.65\% & 3.20\% & 0.136 \\
Qwen3-32B-DS-SFT & 23.56\% & 13.09\% & 0.021 & 13.61\% & 3.40\% & 0.197 \\
\bottomrule
\end{tabular*}
\end{table}
\endgroup

\end{document}